\documentclass{article}

\usepackage[nonatbib,final]{neurips_2021}

\usepackage{hyperref}
\hypersetup{
    colorlinks,
    linkcolor={red!50!black},
    citecolor={green!50!black},
    urlcolor={blue!80!black}
}
\usepackage{amsfonts}
\usepackage{amssymb}
\usepackage{amsmath}
\usepackage{amsthm}
\usepackage{bm}
\usepackage{breakcites}
\usepackage{centernot}
\usepackage[capitalize]{cleveref}
\usepackage{dsfont}
\usepackage[inline, shortlabels]{enumitem}
\usepackage[T1]{fontenc}
\usepackage[utf8]{inputenc}
\usepackage{mathtools}
\usepackage{microtype}     
\usepackage{parskip}
\usepackage{scalerel}
\usepackage{setspace}
\usepackage{upquote}
\usepackage{verbatim}

\usepackage[color=yellow, textsize=tiny]{todonotes}
\usepackage{comment}

\usepackage[backend=bibtex,style=numeric,backref=true,natbib=true]{biblatex}
\DefineBibliographyStrings{english}{%
  backrefpage = {page},
  backrefpages = {pages},
}
\usepackage{chngcntr}
\usepackage{apptools}
\AtAppendix{\counterwithin{theorem}{section}}

\theoremstyle{definition}
\newtheorem{theorem}{Theorem}

\newtheorem{lemma}[theorem]{Lemma}

\newtheorem{corollary}[theorem]{Corollary}
\newtheorem{proposition}[theorem]{Proposition}

\newtheorem*{theorem*}{Theorem}
\newtheorem*{proposition*}{Proposition}
\newtheorem*{definition*}{Definition}

\newtheorem*{question*}{Question}

\newtheorem*{claim*}{Claim}
\newtheorem*{assumption*}{assumption}

\newcommand\argmin{\operatorname*{argmin}}

\DeclarePairedDelimiterX{\twobarparen}[2]{(}{)}{#1\;\delimsize\|\;#2}
\DeclarePairedDelimiterX{\inner}[2]{\langle}{\rangle}{#1, #2}
\DeclarePairedDelimiterX{\binner}[2]{\lparen}{\rparen}{#1 \vert{} #2}

\let\emptyset\varnothing

\newcommand{\dd}{\mathop{}\!\mathrm{d}}
\newcommand{\abs}[1]{\left\vert{} #1 \right\vert{}}

\newcommand{\G}{\mathcal{G}}

\let\P\relax
\DeclareMathOperator{\supp}{\text{supp}}

\DeclareMathOperator{\P}{\mathbb{P}}
\DeclareMathOperator{\E}{\mathbb{E}}

\DeclareMathOperator{\Unif}{\text{Unif}}

\newcommand{\iid}{i.i.d.}

\renewcommand{\H}{\mathcal{H}}

\newcommand{\R}{\mathbb{R}}

\newcommand{\X}{\mathcal{X}}

\newcommand{\norm}[1]{\lVert#1\rVert}
\newcommand{\lnorm}[1]{\left\lVert#1\right\rVert}
\newcommand{\fnorm}[1]{\norm{#1}_\text{F}}

\DeclareMathOperator{\tr}{Tr} % trace
\DeclareUnicodeCharacter{2212}{$-$}

\renewcommand{\O}{\mathcal{O}}

\newcommand{\xx}{\bm{X}}
\newcommand{\yy}{\bm{Y}}
\newcommand{\xxi}{\bm{\xi}}

\newcommand{\Lmu}{{L_2(\X, \mu)}}

\newcommand{\pJ}{J^\perp}

\newcommand{\te}{\tilde{e}}

\newcommand{\lmin}{\lambda_\text{min}}
\newcommand{\lmax}{\lambda_\text{max}}
\DeclareMathOperator{\id}{\text{id}}
\newcommand{\dimeff}{\dim_{\text{eff}}}

\let\origperp\perp
\renewcommand{\perp}{{\scaleobj{0.75}{\origperp}}} %}%{\scaleobj{0.75}{0}}

\newcommand{\Hs}{\overline{\H}}
\newcommand{\Ha}{\H_\perp}

\colorlet{notegreen}{green!50!black}

\newenvironment{supplementary}{}{}
\def\submission{0}

\title{Provably Strict Generalisation Benefit for Invariance in Kernel Methods}

\author{%
    Bryn Elesedy\\%\thanks{\href{https://bryn.ai}{https://bryn.ai}} \\
    University of Oxford\\
    \texttt{bryn@robots.ox.ac.uk}\\
}

\begin{document}

\if\submission1
    \excludecomment{supplementary}
\fi

\maketitle

\begin{abstract}
    It is a commonly held belief that enforcing invariance 
    improves generalisation.
    Although this approach enjoys widespread popularity,
    it is only very recently that a rigorous theoretical demonstration of 
    this benefit has been established.
    In this work we build on the function space perspective 
    of~\citet{elesedy2021provably} to derive a strictly non-zero
    generalisation benefit of incorporating invariance in 
    kernel ridge regression when the target is invariant to 
    the action of a compact group.
    We study invariance enforced by
    feature averaging
    and find that
    generalisation is governed by a notion of
    effective dimension that arises from the interplay between
    the kernel and the group.
    In building towards this result, we find that 
    the action of the group induces an orthogonal decomposition of 
    both the reproducing kernel Hilbert space and its kernel,
    which may be of interest in its own right.
\end{abstract}

\section{Introduction}
Recently, there has been significant interest in models
that are invariant to the action of a group on their inputs.
It is believed that engineering models in this way improves 
sample efficiency and generalisation.
Intuitively, if a task has an invariance,
then a model that is constructed to be invariant ahead of time should require
fewer examples to generalise than one that must learn to be invariant.
Indeed, there are many application domains, such as fundamental physics
or medical imaging, in which the invariance is known a priori~\cite{spencer2020better,winkels20183d}.
Although this intuition is certainly not new (e.g.~\cite{wood1996representation}),
it has inspired much recent work (for instance, see~\cite{zaheer2017deep,lee2019set}).

However, while implementations and practical applications abound, until
very recently a rigorous theoretical justification for invariance was missing.
As pointed out in~\cite{elesedy2021provably}, many prior works
such as~\cite{sokolic2017generalization,sannai2019improved} provide only
worst-case guarantees on the performance of invariant algorithms.
It follows that these results do not rule out the possibility of
modern training algorithms automatically favouring invariant models, 
irrespective of the choice of architecture.
Steps towards a more concrete theory of the benefit of invariance
have been taken by~\cite{elesedy2021provably,mei2021learning}
and our work is a continuation along the path set by~\cite{elesedy2021provably}.

In this work we provide a precise characterisation of the generalisation
benefit of invariance in kernel ridge regression.
In contrast to~\cite{sokolic2017generalization,sannai2019improved},
this proves a \emph{provably strict} generalisation benefit
for invariant, feature-averaged models.
In deriving this result, we provide insights into the
structure of reproducing kernel Hilbert spaces in relation to 
invariant functions that we believe will be useful for
analysing invariance in other kernel algorithms.

The use of feature averaging to produce invariant predictors
enjoys both theoretical and practical success~\cite{lyle2020benefits,foster2020improving}.
For the purposes of this work, feature averaging is defined as 
training a model as normal (according to any algorithm) and then
transforming the learned model to be invariant.
This transformation is done by \emph{orbit-averaging},
which means projecting the model on the space of invariant
functions using the operator $\O$ introduced in~\cref{sec:function-space}.

Kernel methods have a long been a mainstay of machine learning
(see~\cite[Section 4.7]{steinwart2008support} for a brief historical overview).
Kernels can be viewed as mapping the input data into a potentially
infinite dimensional feature space,
which allows for analytically tractable inference with non-linear predictors.
While modern machine learning practice is dominated by neural networks,
kernels remain at the core of much of modern theory.
The most notable instance of this is 
the theory surrounding the \emph{neural tangent kernel}~\cite{jacot2018neural},
which states that the functions realised by an infinitely wide neural network
belong to a reproducing kernel Hilbert space (RKHS) with a kernel determined by the network architecture.
This relation has led to many results on the theory of optimisation
and generalisation of wide neural networks (e.g.~\cite{lee2019wide,arora2019fine}).
In the same vein, via the NTK, we believe the results of this paper 
can be extended to study wide, invariant neural networks.

\subsection{Summary of Contributions}
This paper builds towards a precise characterisation of
the benefit of incorporating invariance in kernel ridge regression by
feature averaging.

\cref{thm:structure}, given in~\cref{sec:structure}, forms the basis of our work, 
showing that the action of the 
group $\G$ on the input space induces an orthogonal decomposition 
of the RKHS $\H$ as 
\[
    \H = \Hs \oplus \Ha
\]
where each term is an RKHS and
$\Hs$ consists of all of the invariant functions in $\H$.
We stress that, while the main results of this paper concern kernel ridge regression,
\cref{thm:structure} holds regardless of training
algorithm and could be used to explore invariance in other kernel methods.

Our main results are given in~\cref{sec:generalisation} and we outline them here.
We define the generalisation gap $\Delta(f, f')$ for two predictors $f, f'$ as the
difference in their test errors. 
If $\Delta(f, f') > 0$ then $f'$ has \emph{strictly better test performance} than
$f$.
\cref{thm:tta} describes $\Delta(f, f')$ for $f$ being the solution to 
kernel ridge regression and $f'$ its invariant (feature averaged) version
and shows that it is positive when the target is invariant.

More specifically, let $X \sim \mu$ where $\mu$ is $\G$-invariant and
$Y = f^*(X) + \xi$ with $f^*$ $\G$-invariant and $\E[\xi] =0$,
$\E[\xi^2] = \sigma^2 < \infty$. Let $f$ be the solution to kernel ridge
regression with kernel $k$ and regularisation parameter $\rho > 0$ on $n$ \iid~training
examples $\{(X_i, Y_i) \sim (X, Y):i=1, \dots, n\}$
and let $f'$ be its feature averaged version.
Our main result, \cref{thm:tta}, says that
\[
    \E[\Delta(f, f')]
    \ge
    \frac{\sigma^2 \dimeff(\Ha) + \mathcal{E}}{(\sqrt{n}M_k + \rho/\sqrt{n})^2}
\]
where $M_k = \sup_x k(x, x) <\infty$, $\mathcal{E} \ge 0$ describes the approximation errors
and $\dimeff(\Ha)$ is the effective dimension of the RKHS $\Ha$.
For an RKHS $\H$ with kernel $k$ the effective dimension is defined by
\[
    \dimeff(\H) = \int_\X k(x, y)^2 \dd\mu(x) \dd\mu(y).
\]
where $\X = \supp \mu$. We return to this quantity at various points in the paper.

It is important to note that the use of the feature averaged predictor $f'$ as a
comparator is without loss of generality. Any other predictor $f''$ 
that has test risk not larger than $f'$ would satisfy the above bound, simply
because this means $\Delta(f', f'') \ge 0$ so 
$\Delta(f, f'') = \Delta(f, f') + \Delta(f', f'') \ge \Delta(f, f')$.\footnote{%
To be completely clear: if, for instance, it so happens that projecting the RKHS onto a space
of invariant predictors before doing KRR gives lower test risk than projecting afterwards
(what we are calling feature averaging), then our result applies in that case too.}

Finally, for intuition, in~\cref{thm:linear}
we specialise~\cref{thm:tta} to the linear setting and compute the bound exactly.
Assumptions and technical conditions are given in~\cref{sec:preliminaries}
along with an outline of the ideas of~\citet{elesedy2021provably} on which we build.
Related works are discussed in~\cref{sec:related}.

\section{Background and Preliminaries}\label{sec:preliminaries}
In this section we provide a brief introduction to reproducing kernel Hilbert spaces
(RKHS)
and the ideas we borrow from~\citet{elesedy2021provably}.
Throughout this paper, $\H$ with be an RKHS with kernel $k$.
In~\cref{sec:technical} we state some topological and measurability
assumptions that are needed for our proofs. 
These conditions are benign and the reader not interested in technicalities
need take from~\cref{sec:technical} only that $\mu$ is $\G$-invariant
and 
that the kernel $k$ is bounded and satisfies~\cref{eq:kernel-switch}.
\if\submission1
We defer proofs to the Supplementary Material.
\else
We defer some background and technical results 
to~\cref{sec:structure-background,sec:technical-results} respectively.
\fi

\subsection{RKHS Basics}
A Hilbert space is an inner product space that is complete
with respect to the norm topology induced by the inner product.
A reproducing kernel Hilbert space (RKHS) $\H$ is Hilbert space of real functions 
$f: \X \to \R$ on which the evaluation
functional $\delta_x : \H \to \R$ with $\delta_x[f] = f(x)$ is continuous $\forall x \in \X$,
or, equivalently is a bounded operator.
The Riesz Representation Theorem tells us that
there is a unique function $k_x \in \H$ such that 
$
    \delta_x[f] = \inner{k_x}{f}_\H
$
for any $f \in \H$,
where $\inner{\cdot}{\cdot}_\H: \H \times \H \to \R$ is the inner product on $\H$.
We identify the function $k : \X \times \X \to \R$ with $k(x, y) = \inner{k_x}{ k_y}_\H$
as the \emph{reproducing kernel} of $\H$.
Using the inner product representation, one can see that $k$ is positive-definite and symmetric.
Conversely, the Moore-Aronszajn Theorem shows that
for any positive-definite and symmetric function $k$, there is a unique
RKHS with reproducing kernel $k$.
In addition, any Hilbert space admitting a reproducing kernel is an RKHS.
Finally, another characterisation of $\H$ is as the completion of 
the set of linear combinations of the form
$
    f_c(x) = \sum_{i=1}^n c_i k(x, x_i)
$
for $c_1, \dots, c_n \in \R$ and $x_1, \dots, x_n \in \X$.
For (many) more details, see~\cite[Chapter 4]{steinwart2008support}.

\subsection{Technical Setup and Assumptions}\label{sec:technical}
\paragraph{Input Space, Group and Measure}
Let $\G$ be a compact%
\footnote{The set of compact groups covers almost all invariances in machine learning, including
all finite groups (such as permutations or reflections), many continuous
groups such as rotations or translations on a bounded domain (e.g.~an image)
and combinations thereof.}%
, second countable, Hausdorff topological
group with Haar measure $\lambda$ (see~\cite[Theorem 2.27]{kallenberg2006foundations}).
Let $\X$ be a non-empty Polish space admitting a finite, $\G$-invariant Borel measure
$\mu$, with $\supp \mu = \X$.
We normalise $\mu(\X) =\lambda(\G) = 1$, the latter 
is possible because $\lambda$ is a Radon measure.
We assume that $\G$ has a measurable action on $\X$ that we will write as $gx$
for $g \in \G$, $x \in \X$.
A measurable action is one such that the map $g: \G \times \X \to \X$ is 
$(\lambda \otimes \mu)$-measurable.
A function $f: \X \to \R$ is $\G$-invariant if
$ f(g x) = f(x)$ $\forall x\in \X$ $\forall g\in \G$.
Similarly, a measure $\mu$ on $\X$ is $\G$-invariant 
if $\forall g \in \G$ and any $\mu$-measurable $B \subset \X$ 
the pushforward of $\mu$ by the action of $\G$ equals $\mu$, i.e.~$(g_* \mu)(B) = \mu(B)$.
This means that if $X \sim \mu$ then $g X \sim \mu$ $\forall g \in \G$.
We will make use of the fact that the Haar measure is $\G$-invariant when $\G$ 
acts on itself by either left or right multiplication,
the latter holding because $\G$ is compact.
Up to normalisation, $\lambda$ is the unique measure on $\G$ with this property.

\paragraph{The Kernel and the RKHS}
Let $k: \X \times \X \to \R$ be a measurable kernel with RKHS $\H$
such that $k(\cdot, x):\X \to \R$ is continuous for any $x \in \X$.
Assume that $\sup_{x \in \X} {k(x, x)} = M_k < \infty$ and note that this implies that
$k$ is bounded since 
\[
    k(x, x') = \inner{k_x}{k_{x'}}_\H \le \norm{k_x}_\H \norm{k_{x'}}_\H 
    = \sqrt{k(x, x)}\sqrt{k(x', x')} \le M_k.
\]
Every $f \in \H$ is $\mu$-measurable, bounded and continuous
by~\cite[Lemmas 4.24 and 4.28]{steinwart2008support} and in addition
$\H$ is separable using~\cite[Lemma 4.33]{steinwart2008support}.
These conditions allow the application of~\cite[Theorem 4.26]{steinwart2008support}
to relate $\H$ to $\Lmu$
in the proofs building towards~\cref{thm:structure},
\if\submission1
given in the Supplementary Material.
\else
given in~\cref{sec:structure-background}.
\fi
We assume that the kernel satisfies, for all $x, y \in \X$,
\begin{equation}\label{eq:kernel-switch}
    \int_\G k(gx, y) \dd \lambda (g) = \int_\G k(x, gy) \dd \lambda (g).
\end{equation}
\Cref{eq:kernel-switch} is our main assumption and we will make frequent use of it.
For~\cref{eq:kernel-switch} to hold, it is sufficient to have $k(gx, y)$ equal to $k(x, gy)$ or $k(x, g^{-1}y)$,
where the latter uses compactness (hence unimodularity) of $\G$ to change variables $g \leftrightarrow g^{-1}$.
Highlighting two special cases:
any inner product kernel $k(x, x') = \kappa(\inner{x}{x'})$
such that the action of $\G$ is unitary with respect to $\inner{\cdot}{\cdot}$
satisfies~\cref{eq:kernel-switch}, as does any stationary kernel
$k(x, x') = \kappa(\norm{x - x'})$ with norm that is preserved by
$\G$ in the sense that
$\norm{gx - gx'} = \norm{x - x'}$ for any $g\in\G$, $x, x' \in \X$.
If the norm/inner product is Euclidean, then any orthogonal representation of $\G$ will
have this property.\footnote{An orthogonal representation of $\G$ on $\R^d$ 
is an action of $\G$ via orthogonal matrices, i.e.~a homomorpishm $\G \to O(d)$.}

\subsection{Invariance from a Function Space Perspective}\label{sec:function-space}
Given a function $f: \X \to \R$ we can define a corresponding orbit-averaged
function $\O f : \X \to \R$ with values
\[
    \O f (x) = \int_\G f(gx) \dd\lambda(g).
\]
$\O f$ will exist whenever $f$ is $\mu$-measurable.
Note that $\O$ is a linear operator and, from the invariance of $\lambda$,
$\O f$ is always $\G$-invariant.
Interestingly, $f$ is $\G$-invariant \emph{only} if $f = \O f$.
\citet{elesedy2021provably} use these observations to 
characterise invariant functions and study their generalisation properties.
In short, this work extends these insights to kernel methods.
Along the way, we will make frequent use of the following (well known) facts about $\O$.
\begin{lemma}[{\cite[Propositions 24 and 25]{elesedy2021provably}}]\label{lemma:basic-properties} A function $f$ is $\G$-invariant if and only if $\O f = f$. This implies that $\O$ is a projection operator, so can have only two eigenvalues $0$ and $1$.
\end{lemma}

\begin{lemma}[{\cite[Lemma 1]{elesedy2021provably}}]\label{lemma:l2-structure}
    $\O:\Lmu \to \Lmu$ is well-defined and self-adjoint.
    Hence, $\Lmu$ has the orthogonal decomposition
    \[
        \Lmu = S \oplus A
    \]
    where $S = \{f \in \Lmu: f\text{ is $\G$ invariant}\}$ and 
    $A = \{f \in \Lmu: \O f = 0\}$. 
\end{lemma}
The meaning of~\cref{lemma:l2-structure} is that any $f\in \Lmu$ has
a (unique) decomposition $f = \bar{f} + f^\perp$ where $\bar{f} = \O f$ is $\G$-invariant
and $\O f^\perp =0$. 
A noteworthy consequence of this setup, as discussed in~\cite{elesedy2021provably},
is a provably non-negative generalisation benefit for feature averaging.
In particular, for any predictor $f \in \Lmu$,
if the target $f^* \in \Lmu$ is $\G$-invariant then the test error
$R(f) = \E_{X \sim \mu}[(f(X) - f^*(X))^2]$ 
satisfies 
\[
    R(f) - R(\bar{f}) = \norm{f^\perp}_\Lmu^2 \ge 0.
\]
The same holds if the target is corrupted by independent, zero mean (additive) noise.
\footnote{The result~\cite[Lemma 1]{elesedy2021provably} is given for equivariance, of which invariance
is a special case.}

\section{Induced Structure of $\H$}\label{sec:structure}
In this section we present~\cref{thm:structure},
which is an analog of~\cref{lemma:l2-structure} for RKHSs.
\Cref{thm:structure} shows that for any compact group $\G$ and RKHS $\H$,
if the kernel for $\H$ satisfies the assumptions in~\cref{sec:technical},
then $\H$ can be viewed as being built from two orthogonal RKHSs,
one consisting of invariant functions and another of those
that vanish when averaged over $\G$.
Later in the paper, this decomposition will allow us to analyse the
generalisation benefit of invariant predictors.

It may seem at first glance that~\cref{thm:structure} should follow immediately from~\cref{lemma:l2-structure},
but this is not the case.
First, it is not obvious that for any $f \in \H$, its orbit averaged version
$\O f$ is also in $\H$.
Moreover, in contrast with $\Lmu$, an explicit form for the inner product on $\H$
is not immediate, which
means that some work is needed to check that $\O$ is self-adjoint on $\H$.
These are important requirements for the proofs of both~\cref{lemma:l2-structure,thm:structure} and
we establish them, along with $\O$ being continuous on $\H$,
in%
\if\submission1
the Supplementary Material.
\else
~\Cref{lemma:o-well-defined,lemma:self-adjoint,cor:o-bounded} respectively.
\fi
The assumption that the kernel satisfies~\cref{eq:kernel-switch} plays a central role.

\begin{lemma}\label{thm:structure}
   $\H$ admits the orthogonal decomposition 
   \[
       \H = \Hs \oplus \Ha
   \]
   where $\Hs  = \{f \in \H : f \text{ is $\G$-invariant}\}$ and
   $\Ha = \{f \in \H: \O f = 0\}$. Moreover, $\Hs$ is an RKHS with
   kernel 
   \[
       \bar{k}(x, y) = \int_\G k(x, gy) \dd\lambda(g)
   \]
   and $\Ha$ is an RKHS with kernel
   \[
       k^\perp(x, y) = k(x, y) - \bar{k}(x, y).
   \]
   Finally, $\bar{k}$ is $\G$-invariant in both arguments.
\end{lemma}
\begin{supplementary}
\begin{proof}
    % make sure this has enough references to relevant preliminary facts.
    From~\cref{lemma:basic-properties} we know that $\O$ is a projection operator.
    Since it is self-adjoint, $\O$ is even an
    %using~\cref{lemma:self-adjoint} we can show that $\O$ is 
    orthogonal projection on $\H$:
    let $h_S$ have eigenvalue $1$ and $h_A$ have eigenvalue $0$ under $\O$, then 
    \[
        \inner{h_S}{h_A}_\H = \inner{\O h_S}{h_A}_\H = \inner{h_S}{\O h_A}_\H = 0.
    \]
    Therefore, by linearity, for any $f \in \H$ we can write 
    $
        f = \bar{f} + f^\perp
        $
    where $\bar{f} = \O f \in \Hs$ is $\G$-invariant and $f^\perp = f - \O f \in \Ha$
    and these terms are mutually orthogonal.
    
    By the linearity of $\O$, it is clear that $\Hs = \O \H$ is an inner product space.
    It is easy to show that $\O$ being continuous implies $\Hs$ is complete.
    Thus $\Hs$ is a Hilbert space, and an RKHS since the evaluation functional
    is clearly continuous on $\Hs \subset \H$.
    For any $h_S \in \Hs$ we have
    \[
        h_S(x) = \inner{h_S}{k_x}_\H = \inner{h_S}{\O k_x}_\H = \inner{h_S}{\bar{k}_x}_\H
    \]
    and the uniqueness afforded by the  
    Riesz representation theorem tells us that the reproducing kernel for $\Hs$
    is $\bar{k}(x, y) = \int_\G k(x, gy)\dd\lambda(g)$.
    We have $\norm{\text{id} - \O} \le 2$ and we can do the same argument to show that 
    $\Ha$ is an RKHS with reproducing kernel $k^\perp$ as claimed.
    Note that one can write $k^\perp(x, y) = \inner{k^\perp_x}{k^\perp_y}_\H$ so it
    must be positive-definite.
    The $\G$-invariance of $\bar{k}(x, y)$ in both arguments is immediate 
    from~\cref{eq:kernel-switch} and~\cref{lemma:basic-properties}.
\end{proof}
\end{supplementary}

As stated earlier, the perspective provided by~\cref{thm:structure} 
will support our analysis of generalisation.
Just as with~\cref{lemma:l2-structure},~\cref{thm:structure} says that
any $f \in \H$ can be written as $f = \bar{f} + f^\perp$ where
$\bar{f}$ is $\G$-invariant and $\O f^\perp = 0$ with 
$\inner{\bar{f}}{f^\perp}_\H = 0$.
As an aside, $\bar{k}$ happens to qualify as a \emph{Haar Integration Kernel},
a concept introduced by~\citet{haasdonk05invariancein}.
We will see that a notion of effective dimension of the
RKHS $\Ha$ with kernel $k^\perp$
governs the generalisation gap between an arbitrary predictor
$f$ and its invariant version $\O f$.
This effective dimension arises from the spectral theory of
an integral operator related to $k$, which we develop in the next section.

\subsection{Spectral Representation and Effective Dimension}\label{sec:spectral}
In this section we consider the spectrum of an integral operator related to the
kernel $k$. 
This analysis will ultimately allow us to define a notion of effective
dimension of $\Ha$ that we will later see is important to
the generalisation of invariant predictors.
While the integral operator setup is standard, the use of this technique to
identify an effective dimension of $\Ha$ is novel.

Define the integral operator $S_k: \Lmu \to \H$ by
\[
    S_k f(x) = \int_\X k(x, x') f(x')\dd \mu (x').
\]
%Understanding $S_k$ will allow us to translate properties of $\O$ on
%$\Lmu$ into properties of $\O$ on $\H$. 
One way of viewing things is that $S_k$ assigns to every element in
$\Lmu$ a function in $\H$.
On the other hand, every $f \in \H$ is bounded so has $\norm{f}_\Lmu < \infty$
and belongs to some element of $\Lmu$.
We write $\iota: \H \to \Lmu$ for the \emph{inclusion map} that sends
$f$ to the element of $\Lmu$ that contains $f$.
In%
\if\submission1
the Supplementary Material
\else
~\cref{lemma:dense-image}
\fi
we show that $\iota$ is injective, so any element of 
$\Lmu$ contains at most one $f \in \H$.

One can define $T_k : \Lmu \to \Lmu$ by $T_k= \iota \circ S_k$,
and~\cite[Theorem 4.27]{steinwart2008support} says that
$T_k$ is compact, positive, self-adjoint and trace-class.
In addition, $\Lmu$ is separable by~\cite[Proposition 3.4.5]{cohn2013measure},
because $\X$ is Polish and $\mu$ is a Borel measure,
so has a countable orthonormal basis.
Hence, by the Spectral Theorem, there exists a countable orthonormal basis
$\{\te_i\}$ for $\Lmu$ such that $T_k \te_i = \lambda_i \te_i$ where 
$\lambda_1 \ge \lambda_2 \ge \dots \ge 0$ are the eigenvalues of $T_k$.
Moreover, since $\iota$ is injective,
for each of the $\te_i$ for which $\lambda_i > 0$ 
there is a unique $e_i \in \H$ such that $\iota e_i = \te_i$ and
$S_k \te_i = \lambda_i e_i$.

Now, since $\iota k_x \in \Lmu$ we have 
\begin{equation}\label{eq:kx-l2}
    \iota k_x 
    = \sum_i \inner{\iota k_x}{\te_i}_\Lmu \te_i  
    = \sum_i (S_k \te_i)(x) \te_i  
    = \sum_i \lambda_i e_i(x) \te_i.
\end{equation}
From now on we permit ourself to drop the $\iota$ to reduce clutter.
We use the above to define
\[
    j(x, y) = \inner{ k_x}{ k_y}_\Lmu, 
    \quad
    \bar{j}(x, y) = \inner{ \bar{k}_x}{ \bar{k}_y}_\Lmu
    \quad \text{and} \quad
   j^\perp(x, y) = \inner{ k^\perp_x}{ k^\perp_y}_\Lmu. 
\]
These quantities will appear again in our analysis of the generalisation of
invariant kernel methods.
Indeed, we will see later in this section that
$\E[j^\perp(X, X)]$ is a type of effective dimension of $\Ha$.
Following~\cref{eq:kx-l2}, one finds the series
representations given below in~\cref{lemma:series}.

The reader may have noticed that our setup is very similar to 
the one provided by Mercer's theorem.
However, we do not assume compactness of $\X$ and so 
the classical form of Mercer's Theorem does not apply.
This aspect of our work is a feature, rather than a bug:
the loosening of the compactness condition allows
application to common settings such as $\X = \R^n$.
For generalisations of Mercer's Theorem see~\cite{steinwart2012mercer} and references therein.

\begin{lemma}\label{lemma:series}
    We have
    \[
        j = \bar{j} + j^\perp.
    \]
    Furthermore, let $\bar{e}_i = \O e_i$ and $e^\perp_i = e_i - \bar{e}_i$ then 
    \[
        j(x, y) =  \sum_i \lambda_i^2 e_i(x) e_i(y),
        \quad
        \bar{j}(x, y) = \sum_i \lambda_i^2 \bar{e}_i(x) \bar{e}_i(y),
        \quad \text{and} \quad
        j^\perp(x, y) = \sum_i \lambda_i^2 e^\perp_i(x) e^\perp_i(y).
    \]
    Finally, the function
    $
        \sum_i \lambda_i^2 \bar{e}_i \otimes e^\perp_i : \X \times \X \to \R
        $
    with values $ (x, y) \mapsto \sum_i \lambda_i^2 \bar{e}_i (x) e^\perp_i(y)$
    vanishes everywhere.
\end{lemma}
\begin{supplementary}
\begin{proof}
    We show in%
    \if\submission1
        the Supplementary Material
        \else
     ~\cref{lemma:commute} 
        \fi
    that $\O$ and $S_k$ commute on $\Lmu$ and $\O$ is self-adjoint on 
    $\Lmu$ by~\cref{lemma:basic-properties}, so $\O$ and $\iota$ 
    (the adjoint of $S_k$ by~\cite[Theorem 4.26]{steinwart2008support})
    must also commute. 
    The first comment is then immediate from the observation that if $a \in \Hs$
    and $b \in \Ha$ one has
    \[
        \inner{\iota a}{\iota b}_\Lmu
         = \inner{\iota \O a}{\iota b}_\Lmu
         = \inner{\O \iota a}{\iota b}_\Lmu
         = \inner{\iota a}{\iota \O b}_\Lmu
         =  0.
    \]
    We also have both of 
    \[
        \inner{\iota \bar{k}_x}{\te_i}_\Lmu = \inner{\iota k_x}{\O \te_i}_\Lmu
        = S_k \O \te_i = \O S_k \te_i = \lambda_i \bar{e}_i
    \]
    and
    \[
        \inner{\iota k^\perp_x}{\te_i}_\Lmu = \inner{\iota k_x}{(\id - \O) \te_i}_\Lmu
        = S_k (\id -  \O) \te_i = (\id -  \O) S_k \te_i = \lambda_i e^\perp_i.
    \]
    Therefore 
    $
        \iota \bar{k}_x = \sum_i \lambda_i \bar{e}_i(x) \te_i
        $
    and
    $
        \iota k^\perp_x = \sum_i \lambda_i e^\perp_i (x) \te_i
        $.
    Taking inner products on $\Lmu$ gives the remaining results.
\end{proof}
\end{supplementary}

Before turning to generalisation, we describe how the above quantities can
be used to define a measure effective dimension. We define
\[
    \dimeff(\H) = \E[j(X, X)]
\]
where $X \sim \mu$. Applying Fubini's theorem, we find
% this is ok, since bounded, positive and measurable (and then series converges so everything is finite).
\begin{equation}\label{eq:effective-dim}
    \dimeff(\H) = \sum_i \lambda_i^2 \E[e_i(X)^2] = \sum_i \lambda_i^2 \norm{\te_i}_\Lmu^2 = \sum_i \lambda_i^2.
\end{equation}
The series converges by the comparison test because $\lambda_i \ge 0$ and $\sum_i \lambda_i = \tr(T_k) < \infty$
(using Lidskii's theorem) because $T_k$ is trace-class.
We have $\dimeff(\H) = \tr(T_k^2)$ and we can think of this (very informally) as taking $\Lmu$, 
pushing it through $\H$ twice using $T_k$ and then measuring its size.
Now because $j = \bar{j} + j^\perp$ we get
\[
    \dimeff(\H) = \dimeff(\Hs) + \dimeff(\Ha)
\]
with 
\[
    \dimeff(\Ha) = \sum_i \lambda_i^2 \norm{\te_i^\perp}_\Lmu^2 = \tr(T_k^2) - \tr((\O T_k)^2)
\]
where $\te_i^\perp = \iota e^\perp_i$.
Again, very informally, this can be thought of as pushing $\Lmu$ through $\Ha$ twice
and measuring the size of the output.
In the next section we will consider the generalisation of kernel ridge regression
and find that $\dimeff(\Ha)$ plays a critical role.

\section{Generalisation}\label{sec:generalisation}
In this section we apply the theory developed in~\cref{sec:structure} to study
the impact of invariance on kernel ridge regression with an invariant target.
We analyse the generalisation benefit of feature averaging, finding
a strict benefit when the target is $\G$-invariant.

\subsection{Kernel Ridge Regression}
Given input/output pairs $\{(x_i, y_i): i=1, \dots, n\}$ where
$x_i \in \X$ and $y_i\in \R$, kernel ridge regression (KRR) returns a predictor
that solves the optimisation problem
\begin{equation}\label{eq:krr}
    \argmin_{f \in \H}C(f) 
    \quad \text{ where } \quad
    C(f) = \sum_{i=1}^n (f(x_i) - y_i)^2 + \rho \norm{f}_\H^2
\end{equation}
and $\rho > 0$ is the regularisation parameter.
KRR can be thought of as performing ridge regression in a
possibly infinite dimensional feature space $\H$.
The representer theorem tells us that the solution to this
problem is of the form
$
    f(x) = \sum_{i=1}^n \alpha_i k_{x_i}(x)
    $
where $\alpha \in \R^n$ solves 
\begin{equation}\label{eq:alpha-ls}
    \argmin_{\alpha \in \R^n} \left\{\norm{\yy - K\alpha}_2^2 + \rho\alpha^\top K\alpha
    \right\},
\end{equation}
$\yy \in \R^n$ is the standard row-stacking of the training outputs with
$\yy_i = y_i$ and
$K$ is the kernel Gram matrix with $K_{ij} = k(x_i, x_j)$.
We consider solutions of the form\footnote{%
    When $K$ is a positive definite matrix this will be the
    \emph{only} solution.
    If $K$ is singular then $\exists c \in \R^n$ with
    $
    \sum_{ij} K_{ij}c_ic_j = \norm{\sum_i c_ik_{x_i}}_\H^2 = 0
        $
    so $\sum_i c_i k_{x_i}$ is identically $0$ and
    $\forall f\in \H$ we get $\sum_i c_i f(x_i) = 0$
    (see~\cite[Section 4.6.2]{manton2014primer}).
    Clearly, this can't happen if $\H$ is sufficiently expressive.
    In any case, the chosen $\alpha$
    is the minimum in Euclidean norm of all possible solutions.
}
$
    \alpha = (K + \rho I)^{-1}\yy
$
which results in the predictor
\[
    f(x) = k_x(\xx)^\top (K + \rho I)^{-1}\yy
\]
where $k_x(\xx) \in \R^n$ is the vector with components $k_x(\xx)_i = k_x(x_i)$.
We will compare the generalisation performance of this predictor with
that of its averaged version
\[
    \bar{f} =  \bar{k}_x(\xx)^\top (K + \rho I)^{-1}\yy \in \Hs.
\]
To do this we look at the generalisation gap.

\subsection{Generalisation Gap}
The generalisation gap is a quantity that compares
the expected test performances of two predictors on a given task.
Given a probability distribution $\P$, data $(X, Y) \sim \P$ and loss function $l$
defining a supervised learning task,
we define the generalisation gap between two predictors $f$ and $f'$ to be
\[
    \Delta(f, f') = \E[l(f(X), Y)] - \E[l(f'(X), Y)]
\]
where the expectations are conditional on the given realisations
of $f, f'$ if the predictors are random.
In this paper we consider $l(a, b) = (a - b)^2$
the squared-error loss and we will assume
$Y = f^*(X) + \xi$ for some target function $f^*$
where $\xi$ is has mean $0$, finite variance and is independent of $X$.
In this case, the generalisation gap reduces to 
\[
    \Delta(f, f') = \E[(f(X) - f^*(X))^2] - \E[(f'(X) - f^*(X))^2].
\]
Clearly, if $\Delta(f, f') > 0$ then we expect strictly better
test performance from $f'$ than $f$.

\subsection{Generalisation Benefit of Feature Averaging}\label{sec:tta}
We are now in a position to give our main result, which is a characterisation
of the generalisation benefit of invariance in kernel methods.
This is in some sense a generalisation of~\cite[Theorem 6]{elesedy2021provably}
and we will return to this comparison later.
We emphasise that~\cref{thm:tta} holds under quite general conditions
that cover many practical applications.

\begin{theorem}\label{thm:tta}
    Let the training data be $\{(X_i, Y_i): i=1, \dots, n\}$ \iid~with
    $
        Y_i = f^*(X_i) + \xi_i
    $
    where $X_i \sim \mu$, $f^* \in \Lmu$ is $\G$-invariant and bounded, and
    $\{\xi_i: i=1, \dots, n\}$ are independent of each other and the $\{X_i\}$,
    with $\E[\xi_i] = 0$ and $\E[\xi_i^2]=\sigma^2 < \infty$.
    Let $f = \argmin_{f \in \H}C(f)$ be the solution to~\cref{eq:krr} and let
    $\bar{f} = \O f \in \Hs$ be the result of applying feature averaging to $f$,
    then the generalisation gap with the squared-error loss satisfies
    \[
        \E[\Delta(f, \bar{f})]
        \ge \frac{ \E[f^*(X)^2 j^\perp(X, X)] + \sigma^2\dim_\text{eff}(\Ha) }{(\sqrt{n}M_k + \rho/\sqrt{n})^2 }
    \]
    where each term is non-negative and
    \[
        \dim_\text{eff}(\Ha) \coloneqq \tr(T_k^2) - \tr((\O T_k)^2)
       = \E[j^\perp(X, X)]
        = \sum_\alpha \lambda^2_\alpha \norm{\te_\alpha^\perp}_\Lmu^2
       \ge 0
    \]
    is the \emph{effective dimension} of $\Ha$. 
\end{theorem}
\begin{supplementary}
\begin{proof}
    Let $\pJ$ be the Gram matrix with components $\pJ_{ij} = j^\perp(X_i, X_j)$
    let $u \in \R^n$ have components $u_i = f^*(X_i)$.
    We can use~\cref{lemma:l2-structure} to get 
    \[
        \Delta(f, \bar{f}) 
        = \E[(k_X^\perp(\xx)^\top(K + \rho I)^{-1}\yy)^2 | \xx, \yy]
    \]
    where $k_x^\perp(\xx) \in \R^n$ with $k_x^\perp(\xx)_i = k_x^\perp(X_i)$.
    Let $\xxi \in \R^n$ have components $\xxi_i = \xi_i$ then
    one finds
    \begin{align*}
        \E[\Delta(f, \bar{f})\vert{} \xx]
        &= \E[(k_X^\perp(\xx)^\top(K + \rho I)^{-1}u)^2 \vert{} \xx]
        + \E[(k_X^\perp(\xx)^\top(K + \rho I)^{-1}\xxi)^2 \vert{} \xx] \\ 
        &= 
         u^\top  (K+\rho I)^{-1}J^\perp (K + \rho I)^{-1}u
         +
        \sigma^2 \tr\left(\pJ(K + \rho I)^{-2}\right)
    \end{align*}
    where the first equality follows because $\xxi$ has mean 0
    and the second comes from the trace trick.

    Consider the first term. We have
    \[
         u^\top  (K+\rho I)^{-1}J^\perp (K + \rho I)^{-1}u 
         = \tr( (K+\rho I)^{-1}J^\perp (K + \rho I)^{-1}u u^\top  ),
    \]
    applying~\cref{cor:trace} twice
    and using~\cref{lemma:op-norm}
    with boundedness of the kernel 
    gives
    \[
        u^\top  (K+\rho I)^{-1}J^\perp (K + \rho I)^{-1}u  
        \ge \lmin((K + \rho I)^{-1})^2 \tr(\pJ uu^\top)
        \ge \frac{u^\top \pJ u}{(M_k n + \rho)^2}
    \]
    so
    \[
        \E[ u^\top  (K+\rho I)^{-1}J^\perp (K + \rho I)^{-1}u  ]
        \ge \frac{\E[u^\top \pJ u]}{(M_k n + \rho)^2}
        = \frac{\sum_{ij} \E[f^*(X_i)f^*(X_j)j^\perp(X_i, X_j)]}{(M_k n + \rho)^2}.
    \]
    For the first term, it remains to show that the above vanishes when $i \ne j$.
    \begin{claim*}
       $\E[f^*(X_i)f^*(X_j)j^\perp(X_i, X_j)] = 0 $ when $i \ne j$.
    \end{claim*}
    \begin{proof}[Proof of claim.]
        Using~\cref{lemma:series} we have
        \[
            \E[f^*(X_i)f^*(X_j)j^\perp(X_i, X_j)] =
            \E\left[\sum_\alpha \lambda_\alpha^2 f^*(X_i) f^*(X_j) e^\perp_\alpha(X_i)  e^\perp_\alpha(X_j)  \right].
        \]
        Define 
        \[
            F_N(X_i, X_j) = \sum_{\alpha=1}^N \lambda_\alpha^2 f^*(X_i) f^*(X_j) e^\perp_\alpha(X_i)e^\perp_\alpha(X_j)
        \]
        then clearly $F_N(X_i, X_j) \to F(X_i, X_j)$ as $N\to\infty$ where 
        \[
            F(X_i, X_j) = \sum_\alpha \lambda_\alpha^2 f^*(X_i) f^*(X_j) e^\perp_\alpha(X_i)e^\perp_\alpha(X_j).
        \]
        On the other hand, since $i\ne j$, the mean of each term is just
        \[
            \E[f^*(X)e^\perp_\alpha(X)]^2 = \inner{\iota f^*}{\te_\alpha^\perp}_\Lmu^2 = 0
        \]
        by the $\G$-invariance of $f^*$ and the orthogonality in~\cref{lemma:l2-structure}.
        It follows by linearity of expectation that $\E[F_N(X_i, X_j)] = 0$ for all $N \ge 0$.
        Now, both $f^*$ and $e^\perp_\alpha$ are bounded so there's a constant $B$ such that
        \[
            \abs{F_N(X_i, X_j)} \le B \sum_{\alpha=1}^{N} \lambda_\alpha^2
            \quad \text{ and } \quad
            \abs{F(X_i, X_j)} \le B \sum_{\alpha} \lambda_\alpha^2
        \]
        and the final sum is finite following the comments after~\cref{eq:effective-dim}.
        We can therefore apply Lebesgue's dominated convergence 
        theorem~\cite[Theorem 1.21]{kallenberg2006foundations} to get that 
        \[
            0 = \E[F_N(X_i, X_j)] \overset{N\to\infty}{\longrightarrow} \E[F(X_i, X_j)] ] = 0 
        \]
        as required.
    \end{proof}

    Moving to the second term, we have again by two applications of~\cref{cor:trace} and 
    then~\cref{lemma:op-norm}
    with boundedness of the kernel that
    \[
        \tr\left(\pJ(K + \rho I)^{-2}\right) 
        \ge \lmin\left( (K+\rho I)^{-2} \right ) \tr(\pJ)
        \ge \frac{\tr(\pJ)}{(M_k n + \rho)^{2}} 
    \]
    and then
    \begin{align*}
        \frac1n \E[\tr(\pJ)]
        &=\frac1n \sum_{i=1}^n \E\left[\sum_\alpha \lambda^2_\alpha e^\perp_\alpha(X_i)e^\perp_\alpha(X_i)\right]  \\ 
        &= \sum_\alpha \lambda^2_\alpha \norm{\te_\alpha^\perp}_\Lmu^2\\
        &= \sum_\alpha \lambda^2_\alpha - \sum_\alpha \lambda^2_\alpha \norm{\O \te_\alpha}_\Lmu^2 \\
        &=  \tr(T_k^2) - \tr(T_k^2 \O)
    \end{align*}
    where we exchange the expectation and sum using Fubini's theorem.
    Considering the sum in the second line, note that
    $\norm{\te_\alpha}_\Lmu^2 = 1 = \norm{\O \te_\alpha}_\Lmu^2 + \norm{\te_\alpha^\perp}_\Lmu^2$
    by~\cref{lemma:l2-structure}
    so the sum converges if $\sum_i \lambda_i^2$ converges, which we establised in the comment 
    after~\cref{eq:effective-dim}.
\end{proof}
\end{supplementary}

\Cref{thm:tta} shows that feature averaging is provably beneficial in terms of
generalisation if the mean of the target distribution is invariant.
If $\H$ contains any functions that are not $\G$-invariant then the 
lower bound is strictly positive.
One might think that, given enough training examples, the solution $f$ to~\cref{eq:krr}
would \emph{learn} to be $\G$-invariant.
\Cref{thm:tta} shows that this cannot happen unless the number of examples
dominates the effective dimension of $\Ha$.

Recall the subspace $A$ in~\cref{lemma:l2-structure}.
The role of $\dimeff(\Ha)$ mirrors that of 
$\dim A$ in~\cite[Theorem 6]{elesedy2021provably} and in the context of the theorem
(linear models) $A$ can be thought of as $\Ha$ when $k$ is the linear kernel.
In this sense~\cref{thm:tta} is a generalisation of~\cite[Theorem 6]{elesedy2021provably}.
It is for this reason that we believe that, although the constant $M_k$ in the
denominator is likely not optimal, the $O(1/n)$ rate that matches~\cite{elesedy2021provably}
is tight.
We leave a more precise analysis of the constants to future work.

The second term in the numerator can be interpreted as quantifying the differences in bias.
One has by the definition of $j^\perp$, that
\begin{equation}\label{eq:bias-terms}
    \E[f^*(X)^2 j^\perp(X, X)] = \int_\X f^*(y)^2 k^\perp(x, y)^2 \dd \mu(x) \dd \mu(y)
\end{equation}
using $j^\perp(x, y) = \int_\X k^\perp(t, x)k^\perp(t, y) \dd \mu(t)$. We also have the following proposition.
\begin{proposition}\label{prop:bias-terms}
    \[
    \int_\X f^*(y)^2 k^\perp(x, y)^2 \dd \mu(x) \dd \mu(y)
    = 
    \int_\X f^*(y)^2 \left( k(x, y)^2 - \bar{k}(x, y)^2 \right)  \dd \mu(x) \dd \mu(y)
    \]
\end{proposition}
\begin{supplementary}
\begin{proof}
    Using $k^\perp = k - \bar{k}$ 
    \begin{align*}
    \int_\X f^*(y)^2 k^\perp(x, y)^2 \dd \mu(x) \dd \mu(y) 
    &= \int_\X f^*(y)^2 k(x, y)^2 \dd \mu(x) \dd \mu(y)\\
    &\phantom{=}- 2\int_\X f^*(y)^2 \bar{k}(x, y) k(x, y) \dd \mu(x) \dd \mu(y) \\
    &\phantom{=}+ \int_\X f^*(y)^2 \bar{k}(x, y)^2 \dd \mu(x) \dd \mu(y)
\end{align*}
    while, since $f^*$ is $\G$-invariant, $\mu$ is $\G$-invariant (by assumption) and
   $\G$ is unimodular (because it is compact),
   \begin{align*}
    \int_\X f^*(y)^2 \bar{k}(x, y) k(x, y) \dd \mu(x) \dd \mu(y)
    &= \int_\X \int_\G f^*(gy)^2 \dd\lambda(g) \bar{k}(x, y) k(x, y) \dd \mu(x) \dd \mu(y) \\ 
    &= \int_\X f^*(y)^2 \int_\G \bar{k}(x, gy) k(x, gy) \dd\lambda(g)  \dd \mu(x) \dd \mu(y)\\
    &= \int_\X f^*(y)^2 \bar{k}(x, y)^2 \dd \mu(x) \dd \mu(y)
   \end{align*}
   where the final line follows because $\bar{k}$ is $\G$-invariant.
\end{proof}
\end{supplementary}

For intuition, we present a simple special case of~\cref{thm:tta}.
In particular, the next result shows that~\cref{eq:bias-terms} reduces 
to an approximation error that is reminiscent of the one 
in~\cite[Theorem 6]{elesedy2021provably} in a linear setting.
For the rest of this section we find it helpful to refer to the action $\phi$ of $\G$ explicitly,
writing $\phi(g)x$ instead of $gx$.
\begin{theorem}\label{thm:linear}
    Assume the setting and notation of~\cref{thm:tta}.
    In addition, let $\X = \mathbb{S}_{d-1}$ be the unit $d-1$ sphere and let $\mu = \Unif(\X)$.
    Let $\G$ act via an orthogonal representation $\phi$ on $\X$ and define
    the matrix 
    $
        \Phi = \int_\G \phi(g) \dd\lambda(g)
        $.
    Let $k(x, y) = x^\top y$ be the linear kernel and suppose $f^*(x) = \theta^\top x$
    for some $\theta \in \R^d$.
    Then the bound in~\cref{thm:tta} becomes
    \[
        \E[\Delta(f, \bar{f})]
        \ge \frac{1}{(\sqrt{n} + \rho/\sqrt{n})^2}
        \left(
            \frac{d - \fnorm{\Phi}^2}{d^2}
        +
        \frac{(d-\fnorm{\Phi}^2) \norm{\theta}_2^2}{d^2(d+2)}
    \right)
    \]
    where $\fnorm{\cdot}$ is the Frobenius norm.
    The first term in the parentheses is exactly $\dimeff(\Ha)$
    and the second term is exactly $\E[f^*(X)^2j^\perp(X, X)]$.
\end{theorem}
\begin{supplementary}
\begin{proof}
    We will make use of the Einstein convention of summing repeated indices.
    Since $\mu$ is finite, by Fubini's theorem~\cite[Theorem 1.27]{kallenberg2006foundations} 
    we are free to integrate in any order throughout the proof.
    First of all notice that $\sup_{x} k(x, x) = 1$ so $M_k = 1$.
    Now observe that
    \[
        \bar{k}(x, y) = x^\top \int_\G \phi(g) y \dd\lambda(g) = x^\top \Phi y.
    \]
    
    Then the first term in the numerator becomes
    \begin{align*}
        \dimeff(\Ha)
        &= \E[j^\perp(X, X)] \\ 
        &= \E[j(X, X)] - \E[\bar{j}(X, X)] \\ 
        &= \int_X k(x, y)^2\dd\mu(x)\dd\mu(y) - \int_X \bar{k}(x, y)^2\dd\mu(x) \dd\mu(y)\\
        &= \int_\X x_ax_b y_ay_b \dd\mu(x) \dd\mu(y)
        - \int_X x_a x_b y_c y_e\Phi_{ac}\Phi_{be} \dd\mu(x) \dd\mu(y) \\ 
        &= \frac1d - \frac{1}{d^2} \fnorm{\Phi}^2.
    \end{align*}
    Where $x_a$ is the $a$\textsuperscript{th} component of $x$, and so on.
    Now for the second term. 
    We calculate each term of the right hand side of~\cref{prop:bias-terms} separately.
    We know that
    \[
        f^*(x)^2 k(x, y)^2
        = (\theta^\top x)^2 (x^\top y)^2
        = \theta_a \theta_b y_c y_e x_a x_b x_c x_e.
    \]
    Integrating $y$ first, we get
    \begin{align*}
        \int_\X f^*(x)^2 k(x, y)^2 \dd \mu(x) \dd\mu(y)
        &= \int_\X \theta_a \theta_b y_c y_e x_a x_b x_c x_e \dd \mu(x) \dd\mu(y)\\ 
        &= \frac1d \int_\X \theta_a \theta_b x_a x_b \dd\mu(x) \\ 
        &= \frac{1}{d^2} \norm{\theta}_2^2
    \end{align*}
    Similarly, we find
    \begin{align*}
        \int_\X f^*(x)^2 \bar{k}(x, y)^2 \dd \mu(x) \dd\mu(y)
        &= \int_\X  \theta_a \theta_b x_a x_b x_c x_e y_f y_h \Phi_{cf}\Phi_{eh} \dd \mu(x) \dd\mu(y)\\
        &= \frac1d \theta_a \theta_b \Phi_{cf}\Phi_{ef} \int_\X  x_a x_b x_c x_e  \dd \mu(x) .
    \end{align*}
    The 4-tensor $\int_\X  x_a x_b x_c x_e  \dd \mu(x)$ is isotropic, so must have the form
    \[
        \int_\X  x_a x_b x_c x_e  \dd \mu(x)
        = \alpha \delta_{ab} \delta_{ce} + \beta \delta_{ac}\delta_{be} + \gamma \delta_{ae} \delta_{bc}
    \]
    (see, e.g.~\citet{hodge1961}). By symmetry and exchangeability we have $\alpha  = \beta = \gamma$.
    Then contracting the first two indices gives
    \[
        \int_\X x_ax_a x_c x_e \dd\mu(x) = \frac1d \delta_{ce} = \alpha(d + 2)\delta_{ce}
    \]
    so $\alpha = \frac{1}{d(d+2)}$ and we end up with
    \[
        \int_\X f^*(x)^2 \bar{k}(x, y)^2 \dd \mu(x) \dd\mu(y)
        =
        \frac{\norm{\theta}_2^2 \norm{\Phi}_F^2 + 2\norm{\Phi \theta}_2^2}{d^2(d+2)} 
        =
        \frac{\norm{\theta}_2^2 (\norm{\Phi}_F^2 + 2)}{d^2(d+2)} 
    \]
    where the second equality comes from
    \[
        \theta^\top\Phi x = \int_\G\theta^\top\phi(g)x \dd \lambda(g)
        = \int_\G f^*(\phi(g) x) \dd \lambda(g)
        = f^*(x) = \theta^\top x
    \]
    for any $x \in \X$.
    Putting everything together gives the result.
\end{proof}
\end{supplementary}

One can confirm that the generalisation gap cannot be negative 
in~\cref{thm:linear} using Jensen's inequality
\[
    \fnorm{\Phi}^2
    = \lnorm{\int_\G \phi(g) \dd\lambda(g)}^2_\text{F}
    \le \int_\G \fnorm{\phi(g)}^2 \dd\lambda(g)
    = \int_\G \tr(\phi(g)^\top\phi(g)) \dd\lambda(g)
    = \tr(I) = d
\]
because the representation $\phi$ is orthgonal.

The matrix $\Phi$ in~\cref{thm:linear} can be computed analytically for 
various $\G$ and in the linear setting describes the importance of the
symmetry to the task.
For instance, in the simple case that $\G = S_d$ the
permutation group on $d$ elements and $\phi$ is the natural representation
in terms of permutation matrices, we have $\Phi = \frac1d \bm{1}\bm{1}^\top$
where $\bm{1} \in \R^d$ is the vector of all 1s. In this case, since
the target is assumed to be $\G$-invariant, we must have $\theta = t\bm{1}$
for some $t \in \R$. Specifically, \Cref{thm:linear} then asserts
\[
    \E[\Delta(f, \bar{f})]
    \ge \frac{(d-1)(dt^2 +d+2)}{d^2(d+2)(\sqrt{n} + \rho/\sqrt{n})^2}.
\]

\section{Related Work}\label{sec:related}
Incorporating invariance into machine learning models is not a new idea.
The majority of modern applications concern neural networks,
but previous works have used kernels~\cite{haasdonk05invariancein,raj2017local}, 
support vector machines~\cite{scholkopf96incorporatinginvariances} and polynomial
feature 
spaces~\cite{schulz1994constructing,schulz1992existence}.
Indeed, early work also considered invariant neural networks~\cite{wood1996representation},
using methods that seem to have been rediscovered in~\cite{ravanbakhsh2017equivariance}.
Modern implementations include invariant/equivariant
convolutional architectures~\cite{cohen2016group,cohen2018spherical}
that are inspired by concepts from mathematical physics and harmonic 
analysis~\cite{kondor2018generalization,cohen2019general}.
Some of these models even enjoy universal approximation 
properties~\cite{maron2019universality,yarotsky2018universal}.

The earliest attempt at theoretical justification for invariance
of which we are aware is~\cite{abu1993hints}, which roughly
states that enforcing invariance cannot increase the VC dimension of a model.
\citet{anselmi2014unsupervised,mroueh2015learning} propose heuristic arguments
for improved sample complexity of invariant models.
\citet{sokolic2017generalization} build on the work of~\citet{xu2012robustness}
to obtain a generalisation bound for certain types of classifiers 
that are invariant to a finite set of transformations,
while~\citet{sannai2019improved} obtain a bound for models that are
invariant to finite permutation groups.
The PAC Bayes formulation is considered in~\cite{lyle2019analysis,lyle2020benefits}.

The above works guarantee only a worst-case improvement and it was not
until very recently that~\citet{elesedy2021provably} derived a strict
benefit for invariant/equivariant models.
Our work is similar to~\cite{elesedy2021provably} in that we provide a 
provably strict benefit, but differs in its application to kernels and RKHSs
as opposed to linear models.
We are careful to state that our setting does not directly reduce to
that of~\cite[Theorem 6]{elesedy2021provably} for two reasons.
First,~\cite[Theorem 6]{elesedy2021provably} considers $\G$ invariant linear models without
regularisation. This may turn out to be accessible by a $\rho \to 0^+$ limit 
(the so called ridgeless limit) of~\cref{thm:tta}.
More importantly, linear regression is equivalent to kernel regression with
the linear kernel. However, the linear kernel can be unbounded (e.g.~on $\R$),
so does not meet our technical conditions in~\cref{sec:technical}.
We conjecture that the boundedness assumption on $k$ can be removed,
or at least with mild care weakened to hold $\mu$-almost-surely.

Also very recently,~\citet{mei2021learning} analyse the
generalisation benefit of invariance in kernels and random feature models.
Our results differ from~\cite{mei2021learning} 
in some key aspects.
First,~\citet{mei2021learning} focus on kernel ridge regression
with an invariant inner product kernel whereas we study
symmetrised predictors from more general kernels.
Second, they obtain an expression
for the generalisation error that is conditional on the training data and
in terms of the projection of the predictor onto a space of high degree
polynomials, while we are able to integrate against the training data
and express the generalisation benefit directly in terms of properties of the kernel
and the group.

\section{Discussion}\label{sec:discussion}
We have demonstrated a provably strict generalisation benefit for feature averaging
in kernel ridge regression.
In doing this we have leveraged an observation on the structure of RKHSs
under the action of compact groups.
We believe that this observation is applicable to other kernel methods too.

There are many possibilities for future work. As we remarked in the introduction,
there is an established connection between kernels and wide neural networks
via the neural tangent kernel. Using this connection, generalisation properties
of wide, invariant neural networks might be accessible through the techniques of this paper.
Another natural extension of this paper is to equivariant (sometimes called \emph{steerable})
matrix valued kernels.
Approximate invariance may be handled by adding an approximation term to the bound in our main result.
Finally, the ideas of this paper should also be applicable to Gaussian processes.

\begin{ack}
 We thank Sheheryar Zaidi for many helpful discussions in the early stages of this
    project and Yee Whye Teh for suggesting the application of~\cite{elesedy2021provably}
    to kernels.
    Additionally, we thank Varun Kanade and Yee Whye Teh for advice and support
    throughout this and other projects.
    This work was supported in part by the UK EPSRC CDT 
    in Autonomous Intelligent Machines and Systems (grant reference EP/L015897/1).
\end{ack}

\begin{supplementary}
\appendix
\section{Notation and Definitions}\label{sec:notation}
Trace of a linear operator $A: V \to V$ on a inner product space $V$ is defined by
\[
    \tr(A) = \sum_i \inner{A v_i}{v_i}
\]
where the collection $\{v_i\}$ forms an orthonormal basis of $V$.
In this paper we will only encounter situations in which the basis
is countable.
This expressions is independent of the basis. We say $A$ is \emph{trace-class}
if $\tr(A) < \infty$.

For any matrix $A \in \R^{n\times n}$, we define $\norm{A}_2 = \sup_{x \in \R^n} \frac{\norm{Ax}_2}{\norm{x}}$
which is the operator norm induced by the Euclidean norm.
For any symmetric matrix $A$, we denote by $\lmax(A)$ and $\lmin(A)$ the largest and smallest
eigenvalues of $A$ respectively.

\section{Useful Results}\label{sec:technical-results}
This section contains some results that are relied upon elsewhere in the paper.

\begin{lemma}[\citet{mori1988trace}]\label{lemma:trace}
    Let $A, B \in \R^{n\times n}$ and suppose $B$ is symmetric.
    Define $A' = \frac12 (A + A^\top)$,
    then
    \[
        \lmin(A')\tr(B) \le \tr(AB) \le \lmax(A') \tr(B),
    \]
    where $\lmin$ and $\lmax$ denote the smallest and largest eigenvalues
    respectively.
\end{lemma}
\begin{corollary}\label{cor:trace}
    Let $A, B \in \R^{n\times n}$ and suppose $A$ is symmetric,
    then
    \[
        \lmin(A)\tr(B) \le \tr(AB) \le \lmax(A) \tr(B).
    \]
\end{corollary}
\begin{proof}
    Let $B' = \frac12 (B + B^\top)$, then using~\cref{lemma:trace} we have
    \[
        \lmin(A)\tr(B') \le \tr(AB') \le \lmax(A)\tr(B').
    \]
    On the other hand, $\tr(B') =  \tr(B)$ and
    \[
        2 \tr(AB') = \tr(AB) + \tr(AB^\top) = \tr(AB) + \tr(BA).
    \]
\end{proof}

\begin{lemma}\label{lemma:op-norm}
    Let $A \in \R^{n\times n}$, then
    \[
        \norm{A}_2 \le n \max_{ij}\abs{A_{ij}}.
    \]
\end{lemma}
\begin{proof}
    Let $a_i \in \R^n$ be the $i$\textsuperscript{th} column of $A$, then
    \[
        \sup_{\norm{x}_2 = 1} \norm{A x}_2  
        =\sup_{\norm{x}_2 = 1} \sqrt{\sum_i (a_i^\top x)^2} 
        \le \sup_{\norm{x}_2 = 1} \sqrt{\sum_i \norm{a_i}_2^2 \norm{x}_2^2}  
        \le \sqrt{\sum_i \norm{a_i}_2^2 }  
        \le \sqrt{n^2 \max_{ij} A_{ij}^2}. 
    \]
    
\end{proof}
\section{Results leading to~\Cref{thm:structure}}\label{sec:structure-background}
Recall from~\cref{sec:structure} the integral operator $S_k: \Lmu \to \H$ defined by
\[
    S_k f (x) = \int_\X k(x, y)f(y) \dd\mu(y)
\]
with adjoint $\iota: \Lmu \to \H$.

\begin{lemma}\label{lemma:dense-image}
    The image of $\Lmu$ under $S_k$ is dense in $\H$
    and $\iota$ is injective.
\end{lemma}
\begin{proof}
    By~\cite[Theorem 4.26]{steinwart2008support} $\norm{f}_\Lmu < \infty$ $\forall f \in \H$
    and $S_k (\Lmu)$ is dense in $\H$
    if and only if the inclusion $\iota: \H \to \Lmu$ is injective.
    Injectivity of the inclusion is equivalent to the statement that
    for any $f, f'\in \H$ the set
    \[
        A(f, f') = \{ x \in \X: f(x) \ne f'(x)\}
    \]
    has $A \ne \emptyset \implies \mu(A) > 0$.
    Continuity implies that for any $f, f' \in \H$, either $f = f'$ pointwise or $A(f, f')$
    contains an open set.
    By the support of $\mu$ this implies $\mu(A) > 0$.
    Thus, $\iota$ is injective.
\end{proof}

From~\cite[Proposition 22]{elesedy2021provably} 
we know that $\O : \Lmu \to \Lmu$ is well-defined and that $\norm{\O} \le 1$.
Let the image of $\Lmu$ under $S_k$ be $\H_2$,
then~\cref{lemma:dense-image} states that $\overline{\H_2} = \H$.

    \begin{lemma}\label{lemma:commute}
        For any $f \in \Lmu$, $\O S_k f = S_k \O f \in \H_2$.
        This implies $\O: \H_2 \to \H_2$ is well defined.
    \end{lemma}
    \begin{proof}
        $\lambda$ is a Radon measure~\cite[Theorem 2.27]{kallenberg2006foundations}
        so is finite because $\G$ is compact
        and all $f \in \H$ are bounded
        so we can apply Fubini's theorem~\cite[Theorem 1.27]{kallenberg2006foundations} as follows:
        taking $f \in \Lmu$ 
        \begin{align*}
            S_k  \O f (x) 
            &= \int_\X \int_\G k(x, y)f(gy) \dd \lambda (g)\dd\mu(y) \\
            &= \int_\X \int_\G k(x, g^{-1}y)f(y) \dd \lambda (g)\dd\mu(y) \quad \text{invariance of $\mu$} \\
            &= \int_\X \int_\G k(gx, y)\dd \lambda (g)f(y) \dd\mu(y) 
            \quad \text{\cref{eq:kernel-switch} then unimodularity of $\G$} \\
            &= \int_\G \int_\X k(gx, y)f(y) \dd\mu(y) \dd \lambda (g)\quad \text{Fubini} \\
            &= \O S_k f (x).
        \end{align*}
        
        Briefly, some detail on the application of Fubini's theorem. Since $f$ may be negative, 
        it is required that
        \[
            \int_\G \int_\X \abs{k(gx, y)f(y)} \dd\mu(y) \dd \lambda (g) < \infty.
        \]
        Observe that
        \begin{align*}
            \int_\G \int_\X \abs{k(gx, y)f(y)} \dd\mu(y) \dd \lambda (g) 
            &= \int_\G \int_\X k(gx, y)\abs{f(y)} \dd\mu(y) \dd \lambda (g) \\
            &= \int_\G S_k\abs{f}(gx) \\
            &\le \lambda(\G) \sup_{x \in \X} S_k\abs{f}(x) <\infty.
        \end{align*}
        The final inequality follows from finiteness of $\lambda$ and the fact that
        $S_k\abs{f} \in \H$ so is bounded.
    \end{proof}

    \begin{lemma}\label{lemma:h2-inner-product}
        Let $a,b \in \H_2$ with preimages $a', b' \in \Lmu$ such that
        $a = S_ka'$ and $b = S_kb'$, then
        \[
            \inner{a}{b}_\H = \int_\X a'(x)b'(y)k(x, y)\dd \mu(x) \dd\mu(y).
        \]
    \end{lemma}
    \begin{proof}
        The inner product on $\H$ is a bounded linear functional, hence commutes with
        integration. We can thus calculate
        \begin{align*}
            \inner{a}{b}_\H 
            &= \inner{\int_\X a'(x)k(x, \cdot)\dd\mu(x)}{\int_\X b'(y)k(y, \cdot)\dd\mu(y)}_\H \\
            &= \int_\X a'(x)b'(y)\inner{k_x}{k_y}_\H \dd\mu(x)\dd\mu(y) \\
            &= \int_\X a'(x)b'(y)k(x, y) \dd\mu(x)\dd\mu(y).
        \end{align*}
    \end{proof}

    \begin{lemma}\label{lemma:h2-self-adjoint}
        For any $f, h \in \H_2$,
    \[
        \inner{\O f}{h}_\H = \inner{f}{\O h}_\H.
    \]
    \end{lemma}
    \begin{proof}
        Let $f'$ and $h'$ be the pre-images of $f$ and $h$ respectively under $S_k$.
        Using~\cref{lemma:h2-inner-product}, Fubini's theorem~\cite[Theorem 1.27]{kallenberg2006foundations},
        the $\G$-invariance of $\mu$
        and~\cref{eq:kernel-switch} we can calculate
        \begin{align*}
            \inner{\O f}{h}_\H 
            &= \int_\X \int_\G f'(gx) h'(y) k(x, y)\dd\lambda(g)\dd\mu(x) \dd\mu(y) \\
            &= \int_\G \int_\X f'(x) h'(y) k(g^{-1}x, y) \dd\mu(x) \dd\mu(y)
                \dd\lambda(g) \quad\text{$\G$-invariance of $\mu$}\\
            &= \int_\G \int_\X f'(x) h'(y) k(x, g^{-1}y) \dd\mu(x) \dd\mu(y) \dd\lambda(g) 
            \quad \text{\cref{eq:kernel-switch}}\\
            &= \int_\G \int_\X f'(x) h'(gy) k(x, y) \dd\mu(x) \dd\mu(y) \dd\lambda(g) 
            \quad \text{$\G$-invariance of $\mu$}\\
            &=\inner{f}{\O h}_\H.
        \end{align*}
        The justification for the application of Fubini's theorem is the same as in
        the proof of~\cref{lemma:commute}.
    \end{proof}

    \begin{lemma}\label{lemma:h2-o-bounded}
        $\O :\H_2 \to \H_2$ is bounded and $\norm{\O}\le 1$.
    \end{lemma}
    \begin{proof}
        Let $f \in \H_2$, then using~\cref{lemma:h2-self-adjoint,lemma:basic-properties}
        along with Cauchy-Schwarz
        \[
        \norm{\O f}_\H^2 = \inner{\O f}{\O f}_\H = \inner{f}{\O f}_\H \le \norm{f}_\H\norm{\O f}_\H.
        \]
    \end{proof}

    \begin{lemma}\label{lemma:o-well-defined}
        $f \in \H \implies \O f \in \H$ so $\O : \H \to \H$ is well defined.
    \end{lemma}
    \begin{proof}
        By~\cref{lemma:dense-image}, for any $f \in \H$
        there is a sequence $\{f_n\} \subset \H_2$ converging to $f$ in $\norm{\cdot}_\H$.
        \Cref{lemma:commute} shows that $\O : \H_2 \to \H_2$ is well defined,
        so the sequence $\{\O f_n\} \subset \H_2$.
        By~\cref{lemma:h2-o-bounded} we have
        $
            \norm{\O f_n - \O f_m}_\H \le \norm{f_n - f_m}_\H
            $
        and so $\{\O f_n\}$ is Cauchy.
        By completeness of $\H$, $\bar{f} \coloneqq \lim_{n \to \infty} \O f_n \in \H$.
        Moreover, $\O$ bounded so is also continuous and we get
        $
            \bar{f} = \lim_{n \to \infty} \O f_n = \O \lim_{n \to \infty} f_n  = \O f
            $.
    \end{proof}

    \begin{lemma}\label{lemma:self-adjoint}
        $\O$ is self-adjoint with respect to the inner product on $\H$.
    \end{lemma}
    \begin{proof}
        We will make use of the continuity of the inner product on $\H$.
        First let $h \in \H$, $f\in \H_2$. We saw from the proof
        of~\cref{lemma:o-well-defined} that $\exists$ sequence $\{h_n\}\subset \H_2$
        with limit $h$ and $\{\O h_n\} \subset \H_2$ with limit $\O h$.
        Then $\inner{\O h_n}{f}_\H \to \inner{\O h}{f}_\H$ and simultaneously,
        applying~\cref{lemma:h2-self-adjoint},
        $
            \inner{\O h_n}{f}_\H = \inner{h_n}{\O f}_\H \to \inner{h}{\O f}_\H
            $
        so the two limits must be equal. Then assuming instead that $f \in \H$
        one can do the same calculation again arrive at the conclusion.
    \end{proof}

    \begin{corollary}\label{cor:o-bounded}
        $\O: \H \to \H$ is bounded with $\norm{\O} \le 1$.
        Indeed, if $\H$ contains any $\G$-invariant functions 
        then $\norm{\O} = 1$ and if not then $\norm{\O}=0$.
    \end{corollary}
    \begin{proof}
        Using~\cref{lemma:self-adjoint} we can repeat the calculation in~\cref{lemma:h2-self-adjoint}.
        The second claim follows from~\cref{lemma:basic-properties} and the variational
        representation of the operator norm.
    \end{proof}

\end{supplementary}

\printbibliography

\if\submission1
    %%%%%%%%%%%%%%%%%%%%%%%%%%%%%%%%%%%%%%%%%%%%%%%%%%%%%%%%%%%%
    \section*{Checklist}

    %%% BEGIN INSTRUCTIONS %%%
    %The checklist follows the references.  Please
    %read the checklist guidelines carefully for information on how to answer these
    %questions.  For each question, change the default \answerTODO{} to \answerYes{},
    %\answerNo{}, or \answerNA{}.  You are strongly encouraged to include a {\bf
    %justification to your answer}, either by referencing the appropriate section of
    %your paper or providing a brief inline description.  For example:
    %\begin{itemize}
    %\item Did you include the license to the code and datasets? \answerYes{See Section~\ref{gen_inst}.}
    %\item Did you include the license to the code and datasets? \answerNo{The code and the data are proprietary.}
    %\item Did you include the license to the code and datasets? \answerNA{}
    %\end{itemize}
    %Please do not modify the questions and only use the provided macros for your
    %answers.  Note that the Checklist section does not count towards the page
    %limit.  In your paper, please delete this instructions block and only keep the
    %Checklist section heading above along with the questions/answers below.
    %%% END INSTRUCTIONS %%%

    \begin{enumerate}

    \item For all authors...
    \begin{enumerate}
    \item Do the main claims made in the abstract and introduction accurately reflect the paper's contributions and scope?
        \answerYes{}
    \item Did you describe the limitations of your work?
        \answerYes{See~\cref{sec:technical} which describes our assumptions.}
    \item Did you discuss any potential negative societal impacts of your work?
        \answerNA{}
    \item Have you read the ethics review guidelines and ensured that your paper conforms to them?
        \answerYes{}
    \end{enumerate}

    \item If you are including theoretical results...
    \begin{enumerate}
    \item Did you state the full set of assumptions of all theoretical results?
        \answerYes{See~\cref{sec:technical} for general technical conditions not given in
            statements of results.}
        \item Did you include complete proofs of all theoretical results?
        \answerYes{Proofs are given in the supplementary material}
    \end{enumerate}

    \item If you ran experiments...
    \begin{enumerate}
    \item Did you include the code, data, and instructions needed to reproduce the main experimental results (either in the supplemental material or as a URL)?
        \answerNA{}
    \item Did you specify all the training details (e.g., data splits, hyperparameters, how they were chosen)?
        \answerNA{}
        \item Did you report error bars (e.g., with respect to the random seed after running experiments multiple times)?
        \answerNA{}
        \item Did you include the total amount of compute and the type of resources used (e.g., type of GPUs, internal cluster, or cloud provider)?
        \answerNA{}
    \end{enumerate}

    \item If you are using existing assets (e.g., code, data, models) or curating/releasing new assets...
    \begin{enumerate}
    \item If your work uses existing assets, did you cite the creators?
        \answerNA{}
    \item Did you mention the license of the assets?
        \answerNA{}
    \item Did you include any new assets either in the supplemental material or as a URL?
        \answerNA{}
    \item Did you discuss whether and how consent was obtained from people whose data you're using/curating?
        \answerNA{}
    \item Did you discuss whether the data you are using/curating contains personally identifiable information or offensive content?
        \answerNA{}
    \end{enumerate}

    \item If you used crowdsourcing or conducted research with human subjects...
    \begin{enumerate}
    \item Did you include the full text of instructions given to participants and screenshots, if applicable?
        \answerNA{}
    \item Did you describe any potential participant risks, with links to Institutional Review Board (IRB) approvals, if applicable?
        \answerNA{}
    \item Did you include the estimated hourly wage paid to participants and the total amount spent on participant compensation?
        \answerNA{}
    \end{enumerate}

    \end{enumerate}

    %%%%%%%%%%%%%%%%%%%%%%%%%%%%%%%%%%%%%%%%%%%%%%%%%%%%%%%%%%%%
\fi

\end{document}